%% file: main.tex
\documentclass[10pt,twocolumn,letterpaper]{article}

\usepackage[pagenumbers]{cvpr} 

\input{preamble}

\definecolor{cvprblue}{rgb}{0.21,0.49,0.74}
\usepackage[pagebackref,breaklinks,colorlinks,allcolors=cvprblue]{hyperref}

\def\paperID{496} 
\def\confName{CVPR}
\def\confYear{2026}

\title{Neurodynamics-Driven Coupled Neural P Systems for Multi-Focus Image Fusion}

\author{Bo Li$^{1,2}$\quad
            Yunkuo Lei$^{1,2}$\quad
            Tingting Bao$^{1,2}$\quad
            Hang Yan$^{1,2}$ \\
            Yaxian Wang$^{4}$\quad
            Weiping Fu$^{1,2}$\quad
            Lingling Zhang$^{1,2}$\thanks{Corresponding author.}\quad
            Jun Liu$^{1,3}$\\
		$^{1}$ School of Computer Science and Technology, Xi’an Jiaotong University \\
         $^{2}$ Ministry of Education Key Laboratory of Intelligent Networks and Network Security, China\\
        $^{3}$ Shaanxi Province Key Laboratory of Big Data Knowledge Engineering\quad
            $^{4}$ Chang’an University\\
		  {\tt\small morvanli@stu.xjtu.edu.cn, zhanglling@mail.xjtu.edu.cn}
          }
\begin{document}
\maketitle
\input{sec/0_abstract}    
\input{sec/1_intro}
\input{sec/2_related_work}
\input{sec/3_cnp}
\input{sec/4_method}
\input{sec/5_experiments}
\input{sec/6_conclusion}
{
    \small
    \bibliographystyle{ieeenat_fullname}
    \bibliography{main}
}


\end{document}

%% file: preamble.tex
\usepackage{algorithm}
\usepackage{algorithmic}
\usepackage{enumitem} 
\usepackage{amsmath} 
\usepackage{booktabs}
\usepackage{amsmath}
\usepackage{amsthm}
\usepackage{amsfonts}
\usepackage{color}
\usepackage{multirow}
\usepackage[table]{xcolor}

\usepackage{pifont}
\newcommand{\cmark}{\ding{51}}%
\newcommand{\xmark}{\ding{55}}%
\usepackage{nicefrac}
\newtheorem{Thm}{Theorem}

\newtheorem{corollary}{Corollary}
\newcommand{\graytext}[1]{{\scriptsize{\textcolor{gray}{#1}}}}

\usepackage{enumitem}

\definecolor{secondcolor}{RGB}{180,198,231}
\definecolor{firstcolor}{RGB}{240,220,220}

\newcommand{\bfsection}[1]{\vspace*{0.00cm}\noindent\textbf{#1.}}



%% file: sec/0_abstract.tex
\begin{abstract}
Multi-focus image fusion (MFIF) is a crucial technique in image processing, with a key challenge being the generation of decision maps with precise boundaries. However, traditional methods based on heuristic rules and deep learning methods with black-box networks are difficult to generate high-quality decision maps.
To overcome this challenge, we introduce neurodynamics-driven coupled neural P (CNP) systems, which are biological neural computation models inspired by spiking mechanisms, to enhance the accuracy of decision maps.
Specifically, we first conduct an in-depth analysis of the model's neurodynamics to identify the constraints between the network parameters and the input signals. This solid analysis avoids abnormal continuous firing of neurons and ensures the model accurately distinguishes between focused and unfocused regions, generating high-quality decision maps for MFIF.
Based on this analysis, we propose a \textbf{N}eurodynamics-\textbf{D}riven \textbf{CNP} \textbf{F}usion model (\textbf{ND-CNPFuse}) tailored for the challenging MFIF task.
Unlike current ideas of decision map generation, ND-CNPFuse distinguishes between focused and unfocused regions by mapping the source image into interpretable spike matrices. By comparing the number of spikes, an accurate decision map can be generated directly without any post-processing. 
Extensive experimental results show that ND-CNPFuse achieves new state-of-the-art performance on four classical MFIF datasets, including Lytro, MFFW, MFI-WHU, and Real-MFF.
The code is available at \url{https://github.com/MorvanLi/ND-CNPFuse}.
\end{abstract}

%% file: sec/1_intro.tex
\section{Introduction}
Multi-focus image fusion (MFIF)~\cite{9428544} is an important task in computer vision that aims to generate an all-in-focus image by fusing multiple images of the same scene captured at varying focal depths. The fused image improves both perceptual quality and information richness, benefiting downstream vision applications such as action  detection~\cite{wang2021discriminative} and object detection~\cite{li2024samf}. By integrating the focused regions from source images with different focus settings, MFIF effectively alleviates the blurring artifacts introduced by the limited depth of field inherent in imaging systems.

\begin{figure}[t]
	\centering
	\includegraphics[width=\columnwidth]{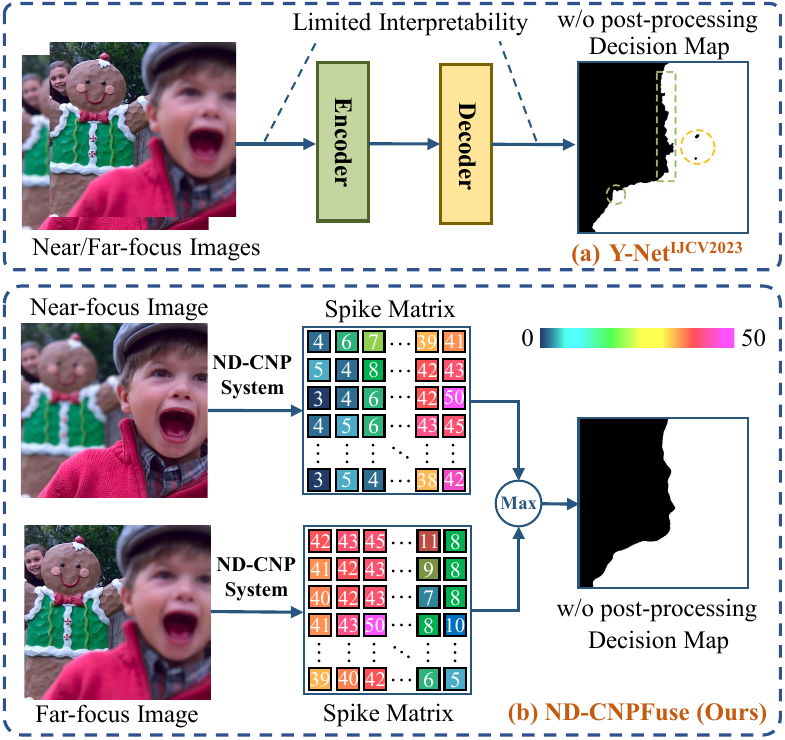}
	\caption{Workflow comparison of the existing decision method and the proposed ND-CNPFuse.
    ND-CNPFuse encodes source images into spike matrices and generates the decision map by comparing spike counts.
   Focused regions produce more spikes than unfocused ones, consistent with human visual perception~\cite{jin2021analyze}.
    }
    \label{fig:workflow}
    \vspace{-1.5em}
\end{figure}
Before the advent of deep learning (DL), MFIF methods were performed in the transform~\cite{liu2015general} and spatial domains~\cite{zhang2017boundary}.
These methods often rely heavily on manually designed fusion rules and strong prior assumptions, which may not always be applicable or sufficiently accurate.

In recent years, DL-based approaches have become the dominant paradigm in MFIF, achieving remarkable progress~\cite{liu2020multi}.
Some studies have applied DL techniques to end-to-end image fusion frameworks~\cite{zhang2021sdnet,10145843,10794610,11208806,guan2023mutual}. Different end-to-end methods may use various network structures, but they generally consist of an encoder network for extracting image features, a fusion network for merging features, and a decoder network for reconstructing the fused features. 
Another line of studies draws inspiration from spatial domain methods, suggesting the use of DL techniques to generate decision maps~\cite{9178463,9805468, qiao2023boosting, 9647937,hu2023zmff}. This process involves training a network to map the source image to a focus map, ultimately producing decision maps. 
Compared to end-to-end methods, decision-based methods directly retain information from the source image at the pixel level, avoiding the spatial consistency issues of end-to-end methods. 
Therefore, more research has been done on DL-based decision map methods.
However, DL-based decision map methods still have evident drawbacks. The most crucial one is that the internal working mechanism is difficult to explain, hindering a clear understanding of the mapping between source images and the decision map~\cite{wang2023multi}.
This will result in nuisance false edges and burrs in the generated decision map, as can be intuitively observed in \cref{fig:workflow}(a).

Coupled neural P (CNP) systems~\cite{8502927} are a class of neural-like computing models
inspired by synchronous pulse bursts in a mammal’s visual cortex.
A CNP system generates synchronous spikes through the mutual coupling of neurons, where synchronous pulses of neurons are treated as 0 or 1.
After multiple iterations, the system produces a sequence of binary images that record the firing states of each neuron at every time step.
By comparing the spike counts of neurons, the degree of focus corresponding to each pixel can be effectively estimated, enabling the differentiation between focused and defocused regions.
\emph{This is particularly well-suited for decision map generation in the MFIF task, which serves as the core motivation for introducing CNP systems into MFIF.}
However, when CNP systems are directly applied to MFIF, neurons may suffer abnormal continuous firing. In this case, the spike counts cannot accurately reflect the focus differences of the source images, which affects the quality of the decision map.
This issue stems from an inappropriate constraint between network parameters and input signals.

To address this problem, we first conduct a detailed study on the neurodynamics of CNP neurons. 
This analysis reveals the underlying constraints between network parameters and input signals, thereby preventing abnormal continuous firing in neurons.
A solid understanding of the neural firing mechanism is helpful for setting appropriate parameters and provides a theoretical basis for effectively applying CNP systems in MFIF.
Building on the neurodynamics-driven CNP system, we then propose a new fusion method called \textbf{ND-CNPFuse}. 
As illustrated in \cref{fig:workflow}(b), ND-CNPFuse provides an interpretable and accurate decision map generation workflow, in contrast to existing DL-based MFIF methods that rely on black-box networks (\cref{fig:workflow}(a)).

In conclusion‌, this paper makes three contributions:

\begin{enumerate}

\item{We conduct a detailed analysis of the neurodynamic mechanisms of CNP neurons, which is essential to clarify the constraints between network parameters and input signals. To the best of our knowledge, this is the first study to investigate the neurodynamics of CNP systems.}

\item{ 
We propose a bio-inspired neural network-based method called ND-CNPFuse to address the challenge of generating accurate decision maps in MFIF tasks.
The advantages of ND-CNPFuse are that it requires no training and the process of generating decision maps is interpretable.
}

\item{Extensive experimental results show that ND-CNPFuse achieves state-of-the-art performance on four MFIF datasets.
Furthermore, ND-CNPFuse consistently outperforms baseline CNP in six commonly used quantitative metrics, with an average improvement of 5.73\%.

}
\end{enumerate}

%% file: sec/2_related_work.tex
\section{Related Work}
\subsection{DL-based MFIF methods}
\bfsection{End-to-end methods} These methods directly generate fused images by leveraging the fitting ability of deep neural networks. 
Representative approaches include convolutional neural networks~\cite{zhang2020ifcnn,9151265,zhang2020rethinking,10323520}, generative adversarial networks~\cite{zhang2021mff,8625482}, Transformer architectures~\cite{9812535,zhu2024task}, and diffusion models~\cite{cao2024conditional,11162636,zhang2024text}.
However, maintaining spatial consistency with source images remains challenging.

\bfsection{Decision-map methods} 
These methods formulate MFIF as a pixel-level classification task by predicting decision maps for focused and defocused regions.
Liu \etal~\cite{liu2017multi} first applied this technique to MFIF, but the generated decision maps remain limited in precision.
Subsequent studies improved performance using ensemble learning~\cite{amin2019ensemble}, attention mechanisms~\cite{9242278}, regression pair learning~\cite{9020016}, multi-scale architectures~\cite{liu2021multiscale}, edge preservation~\cite{wang2023multi}, and defocus blur modeling~\cite{quan2025multi}.
Nevertheless, unclear network mechanisms often introduce pseudo-edges and burr artifacts, and generating high-quality decision maps remains to be explored.

\subsection{Coupled neural P systems}
Coupled neural P (CNP) systems are a class of multi-parameter, single-layer, locally connected neural networks. Their key feature, the coupled spiking mechanism, enables efficient image feature extraction and has been applied in image fusion tasks~\cite{li2021medical,peng2021multi,10399820}. 
However, existing CNP systems heavily rely on manually tuned parameters. This may lead to neurons in invalid states, thus limiting the model's performance in image fusion tasks.
In this paper, we analyze the dynamics of CNP neurons and establish the constraint relationships between network parameters and input signals, which can make the CNP system easier to use.

%% file: sec/3_cnp.tex
\section{Neurodynamics-Driven CNP Systems}
In this section, we first review the definition of the CNP system and then analyze its neurodynamic properties in detail.
\subsection{Mathematical model of CNP systems}
In CNP systems, each neuron has three memory units: the feeding input unit ($U$), the linking input unit ($V$), and the dynamic threshold unit ($T$). The update rule for CNP neuron $\sigma_{ij}$ states is defined as follows:
\begin{equation}
\resizebox{\linewidth}{!}{$
U_{ij}(t)=\left\{
\begin{array}{ll}
     \!\!U_{ij}(t-1) - \!u +\! I_{ij} +\!\!\!\!\! \sum\limits_{\sigma_{kl}\in \delta_r}\!\!\!W_{kl} P_{kl}(t-1), & \text{if}~\sigma_{ij}~\text{fires} \\
    \!\!U_{ij}(t-1) +  I_{ij}+\sum\limits_{\sigma_{kl}\in \delta_r}W_{kl} P_{kl}(t-1), & \text{otherwise} \\
\end{array}
\right.
$}
\end{equation}

\begin{equation}
\resizebox{\linewidth}{!}{$
V_{ij}(t)=\left\{
\begin{array}{ll}
     V_{ij}(t-1) - v + \sum\limits_{\sigma_{kl}\in \delta_r}W_{kl} P_{kl}(t-1), & \text{if}~ \sigma_{ij} ~\text{fires} \\
    V_{ij}(t-1) + \sum\limits_{\sigma_{kl}\in \delta_r}W_{kl} P_{kl}(t-1), & \text{otherwise} \\
\end{array}
\right.
$}
\end{equation}
\begin{flalign}
& T_{ij}(t)=\left\{
\begin{array}{ll}
    T_{ij}(t-1) - \tau + \lambda P_{ij}(t-1), & \text{if}~ \sigma_{ij} ~\text{fires} \\
    T_{ij}(t-1), & \text{otherwise}
\end{array}
\right. &
\end{flalign}
where subscripts $\left ( i,j \right )$ denote the coordinates of the neuron, $\left ( k,l \right )$ denote the coordinates of its neighboring neurons, and $t$ represents the current iteration count. 
$u, v, \tau$ indicate the spikes consumed by three units.
$I_{ij}$ is an external input.
$W_{kl}$ characterizes synaptic weight. 
$P_{kl}$ symbolizes spikes received from the surrounding neurons $\sigma_{kl}$, where neighborhood radius is $r$ ($\sigma_{kl} \in \delta_r$). $\lambda$ is a threshold weight.

\subsection{Why neurodynamics analysis is necessary?}
In certain cases, CNP neurons may enter a state of \emph{continuous firing}, generating spikes at each time step.
As illustrated in \cref{fig:firing}(c) and \cref{fig:firing}(d).
This uncontrolled spiking behavior can lead to model failure, as the output no longer reflects the features of the input signal. For the MFIF task, such outputs do not effectively distinguish between focused and defocused regions, thereby affecting the generation of accurate decision maps. 
We next analyze this invalid state and identify parameter configurations that prevent it.


\bfsection{Analysis of continuous firing condition} We analyze the continuous firing condition of CNP neurons to derive constraints between parameters and the input signal.
For continuously firing CNP neurons, the state update is given by:
\begin{align}
U(t) &= U(t-1) - u + I + \sum\limits_{\sigma_{kl} \in \delta_r} W_{kl} P_{kl}(t-1), \label{eq:U} \\
V(t) &= V(t-1) - v + \sum\limits_{\sigma_{kl} \in \delta_r} W_{kl} P_{kl}(t-1), \label{eq:v} \\
T(t) &= T(t-1) - \tau + \lambda P(t-1). \label{eq:T}
\end{align}

Following~\cite{8502927}, we approximate the spike consumption of the three CNP units as linear decay to facilitate analysis.



\begin{equation}
\begin{aligned}
U(t) - u = \alpha U(t),
V(t) - v = \beta V(t),
T(t) - \tau= \gamma T(t),
\end{aligned}
\end{equation}
where $\alpha, \beta, \gamma \in \left ( 0,1 \right )$. We have the following theorem:

\begin{figure*}[t]
	\centering
	\includegraphics[width=0.85\textwidth]{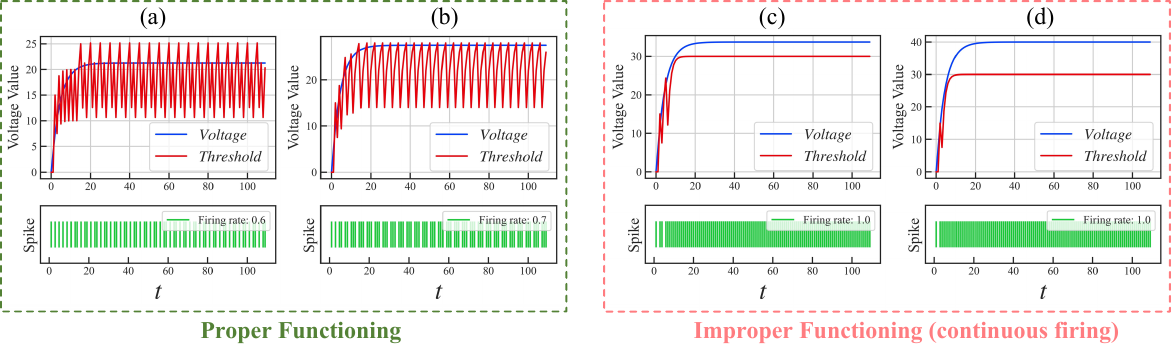}
	\caption{CNP neuron firing states under different inputs. (a) $I=0.5$; (b) $I=1.0$; (c) $I=1.5$; and (d) $I=2.0$.
When $I$ exceeds the continuous firing condition, the neuron exhibits an abnormal state, as observed in (c) and (d).
}
        \label{fig:firing}
        \vspace{-0.5em}
\end{figure*}

\begin{Thm}\label{th1}
    For the feeding input unit U, the value increases with iterations and is represented as
\begin{align}
U(t) & = I\frac{1-\alpha^t}{1-\alpha } +\sum_{n=0}^{t-1}K(n)\alpha^{t-n-1}.  
\end{align}
\end{Thm}

\begin{proof}
Let us clarify the initial conditions: \( U(0) \) and \( K(0) \) are both zero. Define $K(n) = \sum_{\sigma_{kl} \in \delta_r} W_{kl} P_{kl}(n)$. Next, we describe \( U(1), U(2), U(3), \ldots, U(t) \) step by step:
\begin{equation}
\begin{aligned}
U(0) & =0 \\
U(1) & =\alpha U(0)+I +K(0)=I \\
U(2) & =\alpha U(1)+I +K(1)=\alpha I+I +K(1)\\
& \vdots \\
U(t) & = I\frac{1-\alpha^t}{1-\alpha } +\sum_{n=0}^{t-1}K(n)\alpha^{t-n-1}.
\end{aligned}
\end{equation}


Theorem~\ref{th1} is proved.
\end{proof}

\begin{Thm}\label{th2}
   For the linking input unit V, the value grows continuously with the increasing number of iterations and can be represented by the following formula:
\begin{align}
V(t) & = \sum_{n=0}^{t-1}K(n)\beta^{t-n-1}.  
\end{align} 
\end{Thm}
\begin{proof} Similarly, the initial state $ V(0) = 0$. We list $V(1)$, $V(2)$, $V(3)$, $V(4)$, ..., and $V(t)$. 
\begin{equation}
\begin{aligned}
V(0) & =0 \\
V(1) & =\beta V(0)+ K(0)=0 \\
V(2) & =\beta V(1)+ K(1)= K(1)\\
V(3) & =\beta V(2) + K(2)=\beta K(1) + K(2)\\
& \vdots \\
V(t) & = \sum_{n=0}^{t-1}K(n)\beta^{t-n-1}.  
\end{aligned}
\end{equation}

Theorem~\ref{th2} is proved.
\end{proof}

\begin{Thm}\label{th3}
 For the dynamic threshold unit T, the value increases with iterations and is expressed as follows
 \begin{align}
T(t) & = \lambda \frac{1-\gamma^{t-1} }{1-\gamma }. 
\end{align} 
\end{Thm}

\begin{proof}
$T$ denotes the dynamic threshold with initial state \( T(0) = 0 \), and $P$ denotes the output with initial state \( P(0) = 0 \).
A continuously firing neuron maintains \( P(t) = 1 \) as the firing condition always holds. Thus, we can get:

\begin{equation}
\begin{aligned}
T(0) & =0 \\
T(1) & =\gamma T(0) + \lambda P(0)=0 \\
T(2) & =\gamma T(1) + \lambda P(1)=\lambda \\
T(3) & =\gamma T(2) + \lambda P(2)=\gamma \lambda + \lambda \\
& \vdots \\
T(t) & = \lambda \frac{1-\gamma^{t-1} }{1-\gamma }.
\end{aligned}
\end{equation}

Theorem~\ref{th3} is proved.
\end{proof}

\begin{figure*}[t]
	\centering
	\includegraphics[width=0.85\textwidth]{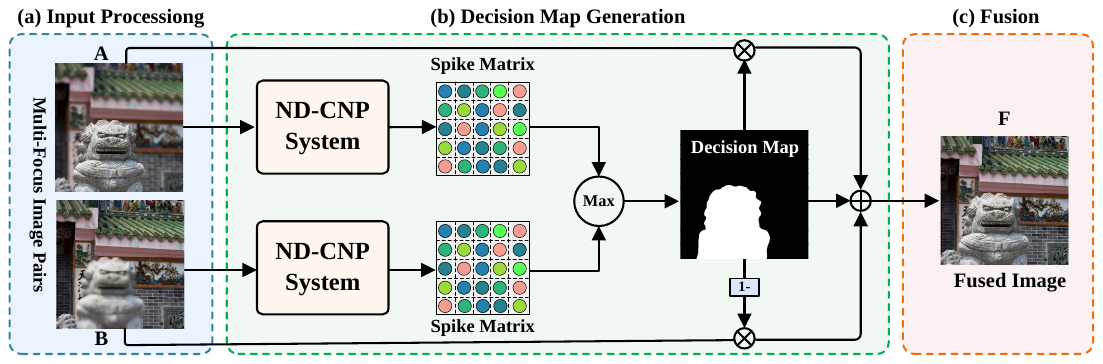}
	\caption{Overview of ND-CNPFuse for MFIF.
    Fused image F quality primarily relies on the decision map generated by the ND-CNP system.
    A and B are the multi-focus source images.
    $\boldsymbol{\otimes}$ and $\boldsymbol{\oplus}$ denote element-wise multiplication and addition.}
        \label{fig:framework}
        \vspace{-0.5em}
\end{figure*}

Based on the above three basic theories, we can further obtain the condition that neurons are in sustained firing, i.e., the following Theorem~\ref{th4}.
\begin{Thm}\label{th4}
Relationship between external input and internal parameters of neurons in the continuous firing state
\begin{equation}
I>\frac{\lambda(1-\alpha)(1-\beta)}{(1-\gamma)(1-\beta+\operatorname{sum}(W))}-\operatorname{sum}(W).
\end{equation}
\end{Thm}

\begin{proof}
    Continuous firing requires the fundamental criterion \(U(t) \cdot (1 + V(t)) > T(t)\) to hold at all times. Substituting results $U(t)$, $V(t)$ and $T(t)$ obtained by Theorems~\ref{th1},~\ref{th2} and~\ref{th3}, we can obtain:
    
    \begin{equation}
\label{eq:relationship}
\begin{split}
& \biggl( I\frac{1-\alpha^t}{1-\alpha} + \sum_{n=0}^{t-1} K(n)\alpha^{t-n-1} \biggl) 
  \biggl( 1 + \sum_{n=0}^{t-1} K(n)\beta^{t-n-1} \biggl) \\
& \quad \quad \quad \quad \quad \quad \quad \quad > 
  \lambda \frac{1-\gamma^{t-1}}{1-\gamma}.
\end{split}
\end{equation}

Since the central CNP neuron and its surrounding neurons all continue to fire, that is, \( P_{kl}(t-1) \equiv 1 \), and we obtain \( K(n) = \mathrm{sum}(W) \), where \( \mathrm{sum}(W) \) denotes the sum of its elements.
Therefore, Eq.~\eqref{eq:relationship} can be expressed as:

\begin{equation}
\resizebox{\linewidth}{!}{$
\begin{split}
& \left ( I\frac{1-\alpha^t}{1-\alpha } + \text{sum}\left ( W \right )\frac{1-\alpha^{t}}{1-\alpha }\right ) \left ( 1+ \text{sum}\left ( W \right )\frac{1-\beta^{t}}{1-\beta }\right )\\
& \quad \quad \quad \quad  \quad \quad \quad  \quad \quad > \lambda \frac{1-\gamma^{t-1} }{1-\gamma }.
\end{split} 
$}
\end{equation}

When $t$ is sufficiently large, the system approaches stability. 
We can obtain Eq.~\eqref{eq:continuous_fire_condition}:

\begin{align}
\label{eq:continuous_fire_condition}
I>\frac{\lambda(1-\alpha)(1-\beta)}{(1-\gamma)(1-\beta+\operatorname{sum}(W))}-\operatorname{sum}(W).
\end{align}

Theorem~\ref{th4} is proved. More detailed derivations are provided in Appendix A.1.
\end{proof}

To verify the accuracy of Theorem~\ref{th4}, we set up a group of values to test: $W = [0.1, 0.5, 0, 0.5, 0.1]$, $\alpha = 0.8$, $\beta = 0.2$, $\gamma = 0.5$, and $\lambda = 15$. The continuous firing condition is calculated based on Eq.~\eqref{eq:continuous_fire_condition}.

\begin{align}
I> \frac{15 \times 0.8 \times 0.2}{(1-0.5 )(1-0.2 +1.2)}-1.2 =1.2.
\end{align}

\cref{fig:firing} shows the firing states of a CNP neuron for different input signals. When external input does not exceed $I$, the neuron can function properly, with firing rates of \textbf{60\%} and \textbf{70\%}, respectively. 
As shown in \cref{fig:firing}(a) and \cref{fig:firing}(b).
In contrast, inputs exceeding $I$ lead to continuous firing with \textbf{100\%} firing rates (\cref{fig:firing}(c) and \cref{fig:firing}(d)). The experimental results are consistent with our theoretical analysis, confirming the correctness of the mathematical derivation.

\begin{corollary}\label{corollary}
According to Theorem~\ref{th4}, the external input to a CNP neuron should remain below its internal state to prevent uncontrolled continuous firing, as formulated below


\begin{align}
\label{eq:no_continuous_fire_condition}
I \le \frac{\lambda (1-\alpha )(1-\beta )}{(1-\gamma )(1-\beta +\operatorname{sum}(W))}-\operatorname{sum}(W).
\end{align}

\end{corollary}

Corollary~\ref{corollary} reveals the necessary constraints between parameters and input signals for the proper functioning of CNP neurons.
Detailed configurations of CNP neuron parameters such as $u$, $v$, and $\tau$ derived from Corollary~\ref{corollary} are provided in Appendix A.2.
Noting that the internal parameters of neurons are automatically configured based on input MFIF images, eliminating the manual tuning, and demonstrating \emph{transferability and generalization across different datasets}.
This helps to improve the ability of CNP systems to distinguish between focused and defocused regions in MFIF tasks, thereby improving the quality of the generated decision map.

%% file: sec/4_method.tex
\section{Methodology}
\subsection{Problem statement}
For a pair of multi-focus images A and B with the same size, MFIF aims to fuse the respective clear regions from images A and B to produce a fully focused image. An efficient and commonly used method is to generate a decision map (DM) of the source images. This pixel-level operation ensures that the fused image F maximizes the retention of information from the source images. As follows:
\begin{equation}
\label{eq:fusion}
\text{F}(i,j) =\text{A}(i,j) \times \text{DM}(i,j) + \text{B}(i,j) \times  (1-\text{DM}(i,j)),
\end{equation}
where $(i,j)$ represents the pixel coordinates in the images.

Eq.~\eqref{eq:fusion} shows that DM accuracy directly impacts fused image quality.
We employ neurodynamics-driven CNP (ND-CNP) systems to generate accurate DM.
Without loss of generality, we also provide details on handling \emph{more than two input images} in Appendix B.

\subsection{Overview of the architecture}
\begin{figure}[t]
	\centering
	\includegraphics[width=\linewidth]{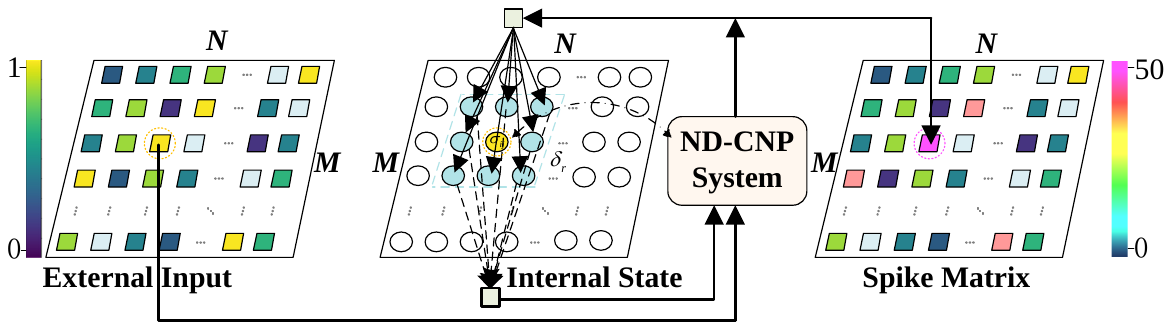}
    \caption{Working principle of the ND-CNP system. $M$ and $N$ denote image height and width.
     Please zoom in for a better view.
    }
   \label{fig:working_principle}
   \vspace{-1em}
\end{figure}

\begin{figure*}[!ht]
\centering
  \includegraphics[width=\linewidth]{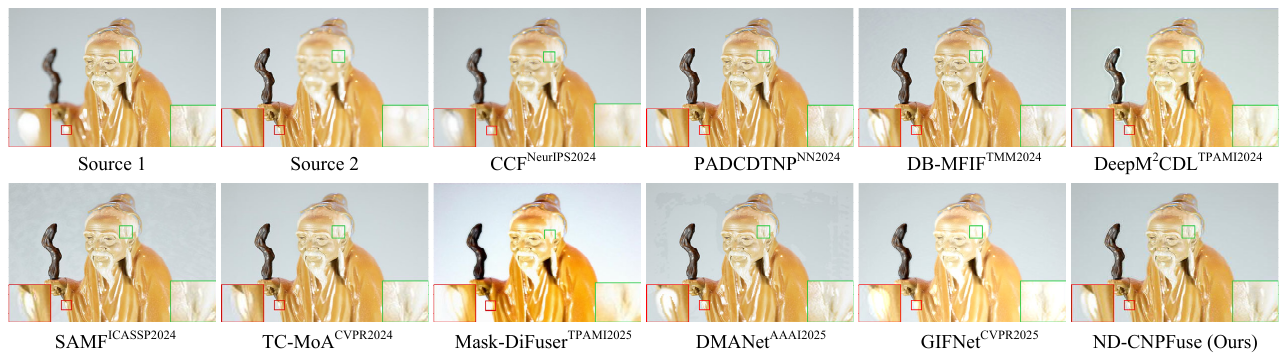}
  \caption{Qualitative results on ``MFFW-4'' from MFFW, with red and green boxes zoomed in 4 times for easy observation.}
  \label{fig:MFFW-4}
\end{figure*}

\begin{figure*}[!ht]
\centering
  \includegraphics[width=\linewidth]{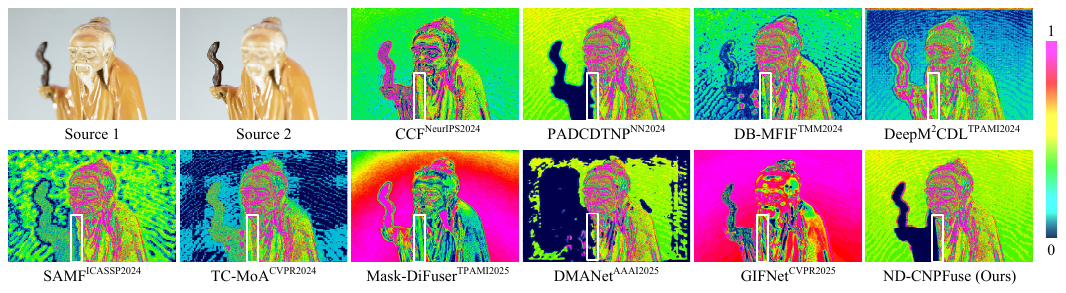}
  \caption{Visualization of difference images, obtained by subtracting Source 1 from the fused results. \textbf{Fewer residuals indicate better information preservation}. White boxes highlight regions for easy comparison. Zooming in for better observation.}
  \label{fig:diff-image}
  \vspace{-1em}
\end{figure*}

The ND-CNPFuse framework is illustrated in \cref{fig:framework}, comprising three components: (a) input processing, (b) decision map generation based on the ND-CNP system, and (c) fusion. The details are described as follows.

\bfsection{(a) Input processing} The pixel values of the source image reflect only simple and low-dimensional information about the image, such as brightness or color. Using pixel values directly as input limits neuron firing. In contrast, image features can provide richer input signals. For this reason, we use sum-modified Laplacian (SML)~\cite{huang2007evaluation} to preprocess the source image.
Further analysis of the impact of SML on the entire framework is detailed in \cref{sec:ablation}.

\bfsection{(b) Decision map generation} 
We use $\Phi_A$ and $\Phi_B$ to represent two ND-CNP systems associated with source multi-focus images A and B. The SML-preprocessed A and B serve as external inputs to $\Phi_A$ and $\Phi_B$, respectively. 
\cref{fig:working_principle} illustrates the detailed working principle of the model.
The systems run from the initial state until reaching the maximum number of iterations. Subsequently, $\Phi_A$ and $\Phi_B$ output two spiking matrices, $SM_A$ and $SM_B$, where each position in $SM_A$ (or $SM_B$) represents the neuron firing counts in $\Phi_A$ (or $\Phi_B$). Note that the mechanism of coupling neurons indicates that neurons exhibit cooperative firing phenomena (with a coupling radius of $r$) rather than independent firing. Therefore, based on $SM_A$ and $SM_B$, the number of firing maps within their coupling radius $r$ is denoted as $F_A$ and $F_B$, respectively. As follows:
\begin{equation}
    \begin{aligned}
& F_{A}(i, j)=\sum_{(x, y) \in r(i, j)}SM_{A}(x,y), \\
& F_{B}(i, j)=\sum_{(x, y) \in r(i, j)}SM_{B}(x,y),
\end{aligned}
\end{equation}
where $r(i,j)$ is the coupling radius centered at position $(i, j)$. Further, we can obtain the DM:
\begin{equation}
\text{DM}(i,j)=\left\{
\begin{array}{ll}
     1, & \text{if}~ F_{A}(i,j)> F_{B}(i,j)  \\
     0, & \text{otherwise} \\
\end{array}.
\right.
\end{equation}

\bfsection{(c) Fusion} After the ND-CNP system generates the DM, we can get the final fused image F according to Eq.~\eqref{eq:fusion}.

\begin{table*}[!htbp]
  \centering
  \caption{Quantitative evaluation results of MFIF tasks. \textbf{Bold} and \underline{underlined} indicate the highest and second-highest scores.}
  \renewcommand{\arraystretch}{1.3} 
  \fontsize{9pt}{10pt}\selectfont
  \setlength{\tabcolsep}{1.3mm}
  \resizebox{\linewidth}{!}{
    \begin{tabular}{@{}lccccccccccccc@{}}
\toprule
                               & \multicolumn{6}{c}{\textbf{Multi-focus Image Fusion on Lytro}}                                                                                                                                                              &                                & \multicolumn{6}{c}{\textbf{Multi-focus Image Fusion on MFFW}}                                                                                                                                                                \\ 
\cmidrule(l){2-7}\cmidrule(l){9-14}
                               & Qabf$\uparrow$                     & FMIw$\uparrow$                   & SSIM$\uparrow$                     & PSNR$\uparrow$                       & QP$\uparrow$                       & QCB$\uparrow$                      &                                & Qabf$\uparrow$                     & FMIw$\uparrow$                   & SSIM$\uparrow$                     & PSNR$\uparrow$                       & QP$\uparrow$                       & QCB$\uparrow$                       \\ 
\midrule
CCF  \hfill\graytext{[NeurIPS2024]}                        & 0.4907                            & 0.3426                            & \underline{0.8525} & 25.1774                           & 0.5868                            & 0.6121                            &                         & 0.4496                          & 0.2866                            & \textbf{0.8384} & 22.6547                           & 0.4320                            & 0.5498                            \\
PADCDTNP  \hfill\graytext{[NN2024]}                      & \underline{0.7613} & \underline{0.5916} & 0.8417                            & 26.9731                            & 0.8394                            & \underline{0.8075} &                       & 0.7367                            & \underline{0.5268} & 0.8224                            & 23.1335  & \underline{0.7104} & 0.7274                             \\
DB-MFIF   \hfill\graytext{[TMM2024]}                      & 0.7479                            & 0.5027                            & 0.8402                            & 26.5203                           & 0.8398                            & 0.7775                            &                        & 0.6958                            & 0.4290                            & 0.8207                            & 23.3339                            & 0.6620                            & 0.6642                             \\
Deep$\mathrm{\text{M}^{2}}$CDL \hfill\graytext{[TPAMI2024]}& 0.7181                            & 0.4623                            & 0.8387                            & 26.0761                           & 0.8092                            & 0.7313                            &  & 0.6718                            & 0.4150                            & 0.8311                            & 23.3246                           & 0.6664                            & 0.6557                             \\
SAMF       \hfill\graytext{[ICASSP2024]}                     & 0.7511                            & 0.5585                            & 0.8401                            & 26.8899                           & 0.8256                            & 0.7951                            &                          & 0.6292                            & 0.3009                            & 0.7984                            & 23.2530                            & 0.5586                            & 0.6693                             \\
TC-MoA  \hfill\graytext{[CVPR2024]}                       & 0.7401                            & 0.5274                            & 0.8415                            & 26.1720                           & 0.8188                            & 0.7597                            &                         & 0.6034                            & 0.2826                            & 0.8216                            & 23.4429                            & 0.5307                            & 0.6475                             \\
Mask-DiFuser  \hfill\graytext{[TPAMI2025]}                       & 0.6004                             & 0.3917                            & 0.8182                             & 25.0254                             & 0.6765                             & 0.5991                             &                        & 0.5585                             & 0.3556                            & 0.7889                             & 22.6972                             & 0.5367                             & 0.5486                              \\
DMANet   \hfill\graytext{[AAAI2025]}                      & 0.7606                            & 0.5807                            & 0.8396                            & \underline{26.9807} & \underline{0.8439} & 0.8055                            &                        & \underline{0.7384} & 0.5247                            & 0.8119                            & \underline{23.6877}                           & 0.7073                            & \underline{0.7279}  \\
GIFNet  \hfill\graytext{[CVPR2025]}                         & 0.5194                            & 0.3268                            & 0.7854                            & 25.9368                            & 0.5412                            & 0.6042                            &                         & 0.4509                            & 0.2524                            & 0.7653                            & 22.5889                            & 0.3850                            & 0.5433                             \\
\rowcolor{gray!15}
\textbf{ND-CNPFuse (Ours)}              & \textbf{0.7621}  & \textbf{0.5967}  & \textbf{0.8541}  & \textbf{26.9904}  & \textbf{0.8466}  & \textbf{0.8092}  &            & \textbf{0.7399}  & \textbf{0.5434}  & \underline{0.8362}  & \textbf{23.7154} & \textbf{0.7294}  & \textbf{0.7291}   \\ 
\midrule
                               & \multicolumn{6}{c}{\textbf{Multi-focus Image Fusion on MFI-WHU}}                                                                                                                                                            &                                & \multicolumn{6}{c}{\textbf{Multi-focus Image Fusion on Real-MFF}}                                                                                                                                                                      \\ 
\cmidrule(l){2-7}\cmidrule(l){9-14}
                               & Qabf$\uparrow$                     & FMIw$\uparrow$                   & SSIM$\uparrow$                     & PSNR$\uparrow$                       & QP$\uparrow$                       & QCB$\uparrow$                      &                                & Qabf$\uparrow$                     & FMIw$\uparrow$                   & SSIM$\uparrow$                     & PSNR$\uparrow$                       & QP$\uparrow$                       & QCB$\uparrow$                       \\ 
\midrule
CCF  \hfill\graytext{[NeurIPS2024]}                       & 0.4972                           & 0.3793                           & \underline{0.8389} & 25.9154                            & 0.5614                          & 0.6691                            &                        & 0.7179                           & 0.4120                           & 0.9458 & 30.0386                           & 0.7743                           & 0.7118                             \\
PADCDTNP \hfill\graytext{[NN2024]}                       & 0.7311                            & 0.6193                            & 0.8316                            & 27.0261                           & 0.7526                            & 0.8231                            &                     & 0.8203                            & \underline{0.6502} & \underline{0.9517}                            & 32.1256                            & \underline{0.9213} & 0.8318                             \\
DB-MFIF  \hfill\graytext{[TMM2024]}                      & 0.7192                            & 0.5278                            & 0.8317                            & \underline{27.6825}                            & 0.7509                            & 0.7813                            &                   & 0.7934                            & 0.5353                            & 0.9481                            & 33.4508                          & 0.8940                            & 0.7990                             \\
Deep$\mathrm{\text{M}^{2}}$CDL \hfill\graytext{[TPAMI2024]}& 0.6793                            & 0.4992                            & 0.8346                            & 26.8603                           & 0.7385                            & 0.7426                            &  & 0.7628                            & 0.5124                            & 0.9384                            & 31.1958                          & 0.8678                            & 0.7784                             \\
SAMF    \hfill\graytext{[ICASSP2024]}                        & 0.7263                            & \underline{0.6248} & 0.8304                            & 27.6122                           & 0.7571                            & 0.8203                            &                           & 0.7093                            & 0.3293                            & 0.9330                            & 30.8761                            & 0.7215                            & 0.7707                             \\
TC-MoA \hfill\graytext{[CVPR2024]}                        & 0.6557                            & 0.5170                            & 0.8388                            & 27.6032                            & 0.7098                            & 0.7660                            &                       & 0.7915                            & 0.5109                            & 0.9501                            & 32.8831                            & 0.8906                            & 0.7971                             \\
Mask-DiFuser \hfill\graytext{[TPAMI2025]}                          & 0.5573                            & 0.4306                            &0.8031                            & 26.3362                            & 0.6230                             & 0.6032                             &                         & 0.6561                           & 0.4534                             & 0.8446                            & 31.5767                            & 0.7927                             & 0.6520                              \\
DMANet    \hfill\graytext{[AAAI2025]}                     & \underline{0.7315} & \underline{0.6248} & 0.8302                            & 27.6611  & \underline{0.7567} & \underline{0.8237} &                        & \underline{0.8205} & 0.6484                            & 0.9507                            & \underline{34.0174} & 0.9194                            & \underline{0.8362}  \\
GIFNet    \hfill\graytext{[CVPR2025]}                     & 0.4388                            & 0.3536                            & 0.7523                            & 26.1136                            & 0.4642                            & 0.5500                            &                        & 0.4834                            & 0.3572                            & 0.8471                            & 30.5346                            & 0.6467                            & 0.6041                             \\
\rowcolor{gray!15}
\textbf{ND-CNPFuse (Ours) }            & \textbf{0.7346}  & \textbf{0.6268}  & \textbf{0.8400}  & \textbf{27.7140} & \textbf{0.7611}  & \textbf{0.8265}  &           & \textbf{0.8207}  & \textbf{0.6666}  & \textbf{0.9569}  & \textbf{34.2024}  & \textbf{0.9237}  & \textbf{0.8398}   \\
\bottomrule
    \end{tabular}
    }
  \label{tab:quantitative}
\end{table*}

\begin{table*}[!ht]
\centering
\caption{Runtime comparison of different fusion methods. PADCDTNP, SAMF, and our method are executed on an i5-13400 CPU, while the remaining methods run on an A100 GPU. Our method reports dual runtimes: 0.41s (MATLAB) and 0.18s (C++) per image pair.}
\fontsize{9pt}{10pt}\selectfont
\setlength{\tabcolsep}{1.4mm} 
 \resizebox{\linewidth}{!}{
\begin{tabular}{lcccccccccc}
\toprule
Method & CCF & PADCDTNP & DB-MFIF & Deep$\mathrm{\text{M}^{2}}$CDL & SAMF & TC-MoA & Mask-DiFuser&DMANet & GIFNet & Ours \\
\midrule
Time (s) & 153.68 & 1.35 & 0.47 & 14.35 & 0.43 & 1.24 &3.61& \underline{0.21} & 0.56& \cellcolor{gray!0}0.41/\textbf{0.18} \\
\bottomrule
\end{tabular}
}
\label{tab:runtime_comparison}
\vspace{-0.9em}
\end{table*}

%% file: sec/5_experiments.tex
\section{Experiments}
\subsection{Experimental configurations}
\bfsection{Dataset} To fully evaluate the effectiveness of NDCNPFuse, we conducted extensive experiments covering four datasets: Lytro~\cite{nejati2015multi}, MFFW~\cite{xu2020towards}, MFI-WHU~\cite{zhang2021mff}, and Real-MFF~\cite{zhang2020real}. 
The Lytro includes 20 pairs of color images. 
The MFFW comprises 13 pairs of color images with defocus spread effects (DSE). 
The MFI-WHU consists of 120 pairs of artificially synthesized color images. 
The Real-MFF contains 710 pairs of real-world MFIF images. 

\bfsection{Comparison methods} We compare several SOTA methods on the MFIF task, including CCF~\cite{cao2024conditional}, PADCDTNP~\cite{li2024multi}, DB-MFIF~\cite{zhang2024exploit}, Deep$\mathrm {M^{2}}$CDL~\cite{10323520}, SAMF~\cite{li2024samf}, TC-MoA~\cite{zhu2024task}, Mask-DiFuser~\cite{11162636}, DMANet~\cite{quan2025multi}, and GIFNet~\cite{Cheng_2025_CVPR}. 
All comparison methods are reproduced using officially provided code and weights.
These methods cover representative end-to-end, decision-based, and spike-based approaches, providing a comprehensive baseline for evaluating fusion performance.

\bfsection{Evaluation metrics}
We use six metrics to evaluate the fusion results from multiple perspectives: edge-based similarity measurement (Qabf), mutual information for wavelet feature (FMIw), structural similarity index (SSIM), peak signal-to-noise ratio (PSNR), phase congruency (QP), and human perception-inspired metric (QCB). For all these metrics, a higher value indicates better fusion quality~\cite{liu2024rethinking}.

\bfsection{Implementation details} ND-CNPFuse runs on an Intel(R) Core(TM) i5-13400 CPU, and DL methods run on an Nvidia A100 GPU. 
We set the coupling radius $r$ to 16 and the number of iterations $t$ to 110, as detailed in \cref{sec:para_analysis}.

\subsection{Performance comparison}
\bfsection{Qualitative comparison}
\cref{fig:MFFW-4} shows the fusion results of different methods on the ``MFFW-4'' sample from the MFFW dataset. 
CCF, PADCDTNP, DB-MFIF, TC-MoA, and our ND-CNPFuse effectively detect the boundary between focused and defocused regions, yielding good visual effects.
However, the other methods encounter various issues.
SAMF and DMANet suffer from edge diffusion artifacts.
Deep$\mathrm {M^{2}}$CDL and GIFNet exhibit chromatic distortions.
Mask-DiFuser loses some spatial details.
The defocus spread effect present in MFFW dataset poses additional challenges for producing high-quality fusion results. Nevertheless, ND-CNPFuse maintains more competitive fusion performance compared to the other methods. 
More qualitative fusion results are provided in Appendix C.1.

We further visualize the difference images to compare fusion results, as shown in \cref{fig:diff-image}. Most methods show visible residuals, while PADCDTNP introduces boundary artifacts. In contrast, ND-CNPFuse fully preserves the focusing information from the source images. 
We also evaluate downstream salient object detection tasks in Appendix C.2.

\bfsection{Quantitative comparison}
\cref{tab:quantitative} reports the quantitative results of all methods on four datasets.
ND-CNPFuse ranks first in five of six evaluation metrics across four datasets.
It falls short of the best SSIM on MFFW as its spike count-based clarity measure may inadequately capture low-contrast edges, slightly impacting edge preservation.
Nevertheless, it still ranks second.  
We acknowledge that the numerical improvement over the second-best method is not significant. However, the contribution of our work lies not only in performance gains but also in offering a novel research perspective for the MFIF field.
We also include the classic spiking neurons LIF model~\cite{fang2023spikingjelly} for comparison, with detailed results provided in Appendix C.3.

\bfsection{Runtime comparison} 
\cref{tab:runtime_comparison} lists the running times of all fusion methods on Lytro.
Our MATLAB implementation achieves 0.41s per image pair. 
By migrating to C++, the runtime is reduced to 0.18s, showing even GPU-accelerated DMANet (0.21s) in CPU-only execution.
This confirms that our approach meets real-time MFIF requirements.

\bfsection{Energy consumption}
Following the energy calculation in~\cite{horowitz20141}, the energy consumption for ND-CNPFuse on the Lytro dataset ($520\times520$) is $1.12\times 10^{-5}$J per image pair.


\subsection{Parameter sensitivity analysis}\label{sec:para_analysis}
We adopt the information retention $Q_{percep}$~\cite{liu2024rethinking} to evaluate the effects of coupling radius and iteration count.

\bfsection{Impact of coupling radius}
We test coupling radius values of 0, 2, 4, 8, 16, and 32. As shown in \cref{fig:parameter_sensitivity}\subref{fig:couple_radius}, $Q_{percep}$ increases with radius and peaks at 16, then declines due to loss of detail and gradient information. Thus, $r$ is set to 16.

\bfsection{Impact of iteration count}
\cref{fig:parameter_sensitivity}\subref{fig:iter_count} shows $Q_{percep}$ results for $t$ values of 50 to 150. Performance improves with more iterations and stabilizes around $t=110$. To balance accuracy and efficiency, we set $t$ to 110.

In summary, the consistent performance across the four datasets demonstrates the universality of $t$ and $r$.

\begin{figure}[t]
	\centering
	\begin{subfigure}[b]{0.49\columnwidth}
		\includegraphics[width=\textwidth]{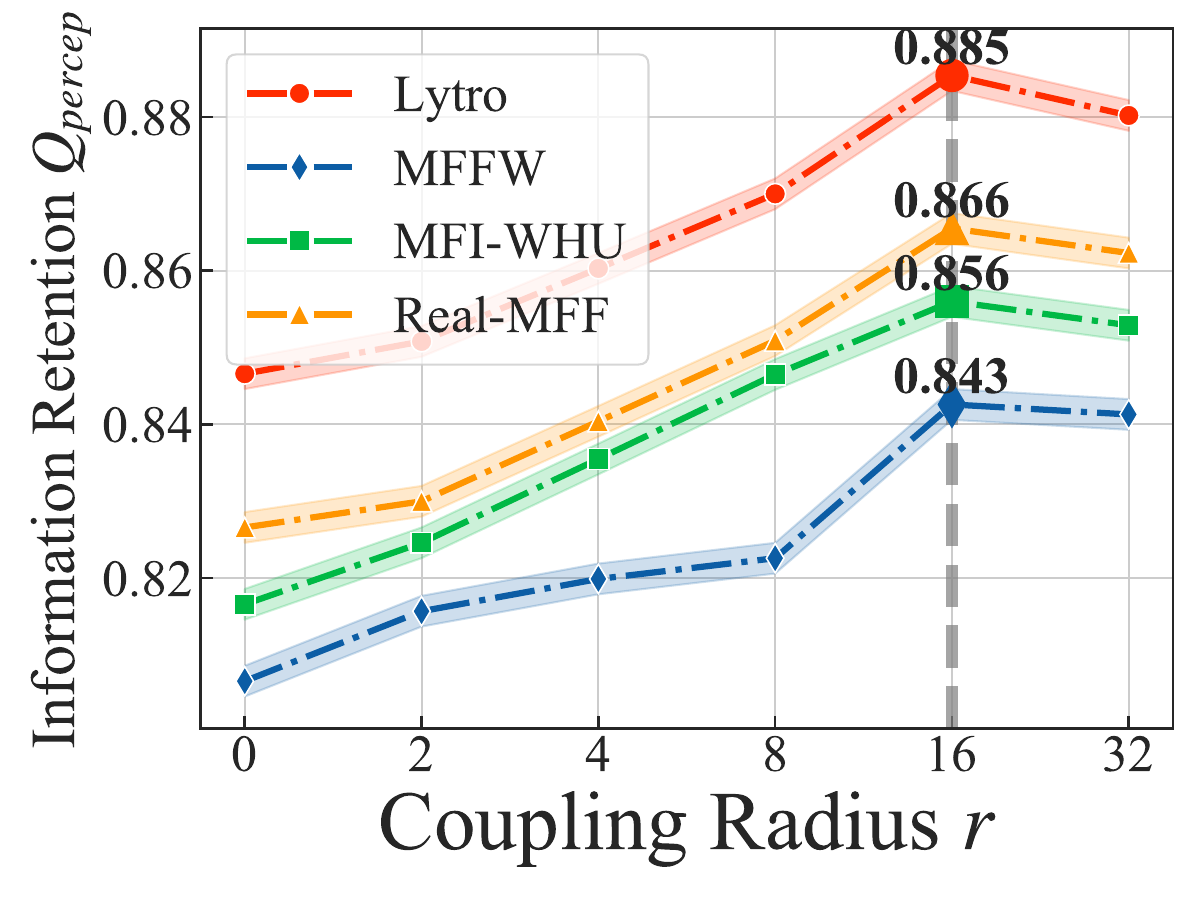}
        \vspace{-1.5em}  
		\caption{}
		\label{fig:couple_radius}
	\end{subfigure}
	\hfill
	\begin{subfigure}[b]{0.49\columnwidth}
		\includegraphics[width=\textwidth]{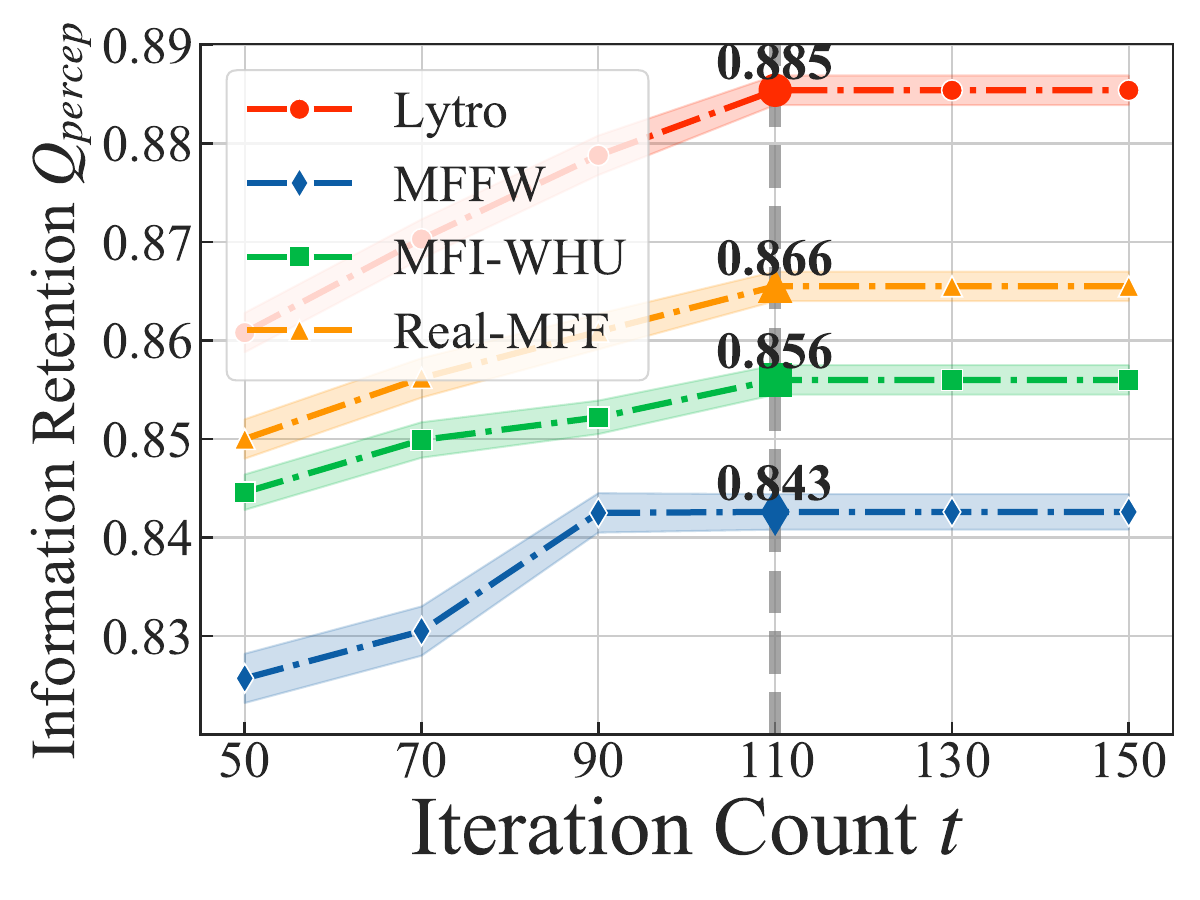}
        \vspace{-1.5em}  
		\caption{}
		\label{fig:iter_count}
	\end{subfigure}
    \vspace{-0.5em}
	\caption{Effect of coupling radius and iteration count on the information retention metric $Q_{percep}$ across four datasets.}
	\label{fig:parameter_sensitivity}
\end{figure}

\begin{table}[t]
\centering
\caption{Ablation study results on the Lytro dataset. \xmark: without neurodynamic analysis, \cmark: with neurodynamic analysis.}
\renewcommand{\arraystretch}{1.15} 
\fontsize{9pt}{10pt}\selectfont
\setlength{\tabcolsep}{1.8mm}  
 \resizebox{\linewidth}{!}{
\begin{tabular}{c|cccccc}
\toprule
Config.  & Qabf$\uparrow$ & FMIw$\uparrow$ & SSIM$\uparrow$ & PSNR$\uparrow$ & QP$\uparrow$ & QCB$\uparrow$ \\
\midrule

\xmark & 0.747 & 0.509 & 0.841 & 25.702 & 0.835 & 0.754 \\
\rowcolor{gray!15}
\cmark & \textbf{0.762} & \textbf{0.597} & \textbf{0.854} & \textbf{26.990} & \textbf{0.847} & \textbf{0.809} \\
$\Delta$~(\%) & {+2.01} & {+17.29} & {+1.55} & {+5.01} & {+1.44} & {+7.29} \\

\bottomrule
\end{tabular}}
\label{tab:ablation_neurodynamic}
\end{table}

\begin{table}[t]
\centering
\caption{Ablation study results on whether SML is used on Lytro.}
\renewcommand{\arraystretch}{1.15} 
\fontsize{9pt}{10pt}\selectfont
\setlength{\tabcolsep}{1.8mm}  
 \resizebox{\linewidth}{!}{
\begin{tabular}{c|cccccc}
\toprule
Config.  & Qabf$\uparrow$ & FMIw$\uparrow$ & SSIM$\uparrow$ & PSNR$\uparrow$ & QP$\uparrow$ & QCB$\uparrow$ \\
\midrule

\xmark~SML& 0.761 & 0.593 & 0.852 & 26.983 & \textbf{0.848} & 0.808 \\
\rowcolor{gray!15}
\cmark~SML& \textbf{0.762} & \textbf{0.597} & \textbf{0.854} & \textbf{26.990} &0.847 & \textbf{0.809} \\
\bottomrule
\end{tabular}}
\label{tab:ablation_SML}
\vspace{-1em}
\end{table}

\subsection{Ablation study}\label{sec:ablation}
\bfsection{Neurodynamic analysis} 
We conducted ablation studies on the Lytro dataset to validate the effectiveness of neurodynamic analysis in improving model performance. 
Specifically, we compare the baseline CNP system and the neurodynamics-driven CNP system across six metrics, as summarized in \cref{tab:ablation_neurodynamic}. 
We can observe that the neurodynamics-driven CNP system achieves significant improvements over the baseline across six multidimensional image quality metrics, with improvements of 2.01\%, 17.29\%, 1.55\%, 5.01\%, 1.44\%, and 7.29\%. 
These findings confirm the effectiveness of neurodynamic analysis.

\bfsection{SML module}
We further validated the impact of the SML module on overall architecture performance. According to the quantitative results in \cref{tab:ablation_SML}, after removing the SML module, evaluation metrics showed only a slight decline. This confirms the generalizability of the proposed method.
Of course, by applying SML, the quality of fusion can be further enhanced.
Moreover, SML is beneficial for handling noisy inputs. 
More details can be found in Appendix C.4.

We believe that ND-CNPFuse is unrestricted in the level of generalizability. 
First, all internal parameters of neurons are automatically configured via neurodynamics. Second, the introduction or removal of the SML module has no significant effect on the overall performance of the framework.
 
\subsection{Case study}
We perform a case study on the Lytro dataset samples. \cref{fig:case_study} shows decision maps from the baseline and ND-CNP systems. The ND-CNP system produces decision maps with clearer boundaries and higher accuracy. In contrast, the original CNP system has uncontrolled spike firing in some areas, resulting in obvious misjudgments in decision maps.

\begin{figure}
	\centering
	\includegraphics[width=\columnwidth]{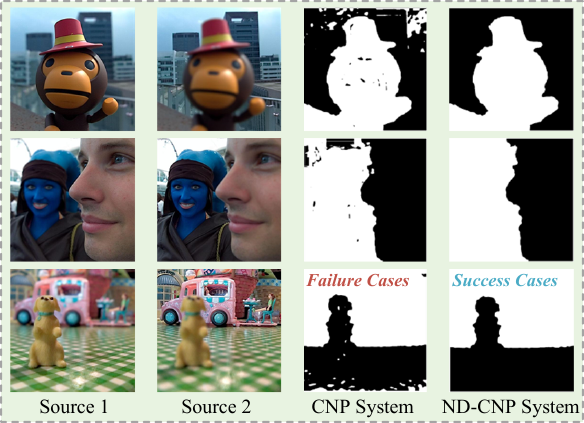}
	\caption{Visualization of decision maps for success cases (ND-CNP system) and failure cases (CNP system) on the Lytro dataset.}
    \label{fig:case_study}
    \vspace{-1em}
\end{figure}


%% file: sec/6_conclusion.tex
\section{Conclusion}
In this paper, we propose a novel MFIF method called ND-CNPFuse based on neurodynamics-driven CNP systems. The main contributions of ND-CNPFuse have two aspects. The first is that we provide a detailed theoretical neurodynamic analysis of the working mechanism of CNP neurons. This analysis reveals the constraints between network parameters and input signals, which helps prevent model failure caused by continuous firing.  
The second one is that we use spike counts of neurons to distinguish between focused and unfocused regions to generate decision maps without any post-processing, and the process is interpretable. 
Extensive experiments with nine SOTA fusion methods on four datasets show that our method is more competitive.

\section*{Acknowledgement}
This work was supported by National Key Research and Development Program of China (No. 2022YFC3303600), the National Natural Science Foundation of China (No. 62137002, 62293553, 62293554, 62437002, 62477036, and 62477037), the ``LENOVO-XJTU'' Intelligent Industry Joint Laboratory Project, the Fundamental Research Funds for the Central Universities (No. xzy022025037), the Shaanxi Provincial Social Science Foundation Project (No. 2024P041), the Natural Science Basic Research Program of Shaanxi (No. 2023-JC-YB-593), the Youth Innovation Team of Shaanxi Universities ``Multi-modal Data Mining and Fusion'', Shaanxi Undergraduate and Higher Education Teaching Reform Research Program (No. 23BY195), Xi'an Jiaotong University City College Research Project (No. 2024Y01), and China Knowledge Centre for Engineering Sciences and Technology.